\documentclass[conference]{IEEEconf}
\IEEEoverridecommandlockouts
\usepackage[noadjust]{cite}
\usepackage{amsmath,amssymb,amsfonts}
\usepackage{algorithmic}
\usepackage{graphicx}
\usepackage{textcomp}
\usepackage[dvipsnames]{xcolor}
\usepackage{dsfont}
\usepackage{amsthm}
\usepackage{xcolor}

\usepackage{mathtools}

\usepackage{hyperref}
\usepackage{cleveref}

\usepackage[normalem]{ulem}

\newtheorem{definition}{Definition}

\let \labelindent \relax
\usepackage{enumitem}
\usepackage{subcaption}
\usepackage[font=small]{caption}

\usepackage{accents}

\newtheorem{theorem}{Theorem}[section]
\newtheorem{corollary}[theorem]{Corollary}
\newtheorem{lemma}[theorem]{Lemma}

\newtheorem{assumption}{Assumption}

\usepackage{bbm}

\def\BibTeX{{\rm B\kern-.05em{\sc i\kern-.025em b}\kern-.08em
    T\kern-.1667em\lower.7ex\hbox{E}\kern-.125emX}}
\begin{document}

\title{
On Dominant Manifolds in Reservoir Computing  Networks 
\thanks{}
}

\author{Noa~Kaplan, Alberto Padoan, Anastasia~Bizyaeva
\thanks{N. Kaplan is with the Department of Computer Science at Cornell University, Ithaca, NY, USA 14853, \tt{nk675@cornell.edu}}
\thanks{ A.~Padoan is with the Department of Electrical and Computer Engineering, University of British Columbia, Vancouver, BC V6T\,1Z4, Canada, \tt{alberto.padoan@ubc.ca}.}%
\thanks{A. Bizyaeva is with the Sibley School of Mechanical and Aerospace Engineering at Cornell University, Ithaca, NY, USA, 14853, \tt{anastasiab@cornell.edu}}
}%

\maketitle

\begin{abstract} 
Understanding how training shapes the geometry of recurrent network dynamics is a central problem in time-series modeling. We study the emergence of low-dimensional dominant manifolds in the training of Reservoir Computing (RC) networks for temporal forecasting tasks. 
For a general linear continuous-time reservoir in the infinite-data limit, we show that the training data generate an invariant subspace of the trained reservoir,
whose dimension equals the number of dominant modes. 
We then specialize to a simplified diagonal linear reservoir, where we link the dominant eigenvalues and eigenvectors 
to the spectrum of a backward Dynamic Mode Decomposition matrix, which yields a finite-dimensional approximation of the backward-time Koopman operator of the system generating the training data. 
%
%
We illustrate the emergence of these dominant modes during training in simulation,  
and 
discuss how the analysis may be extended 
to nonlinear RC via tangent dynamics and differential $p$-dominance.
\end{abstract}


\section{Introduction}

Understanding how training shapes neural network dynamics is a key problem in machine learning and control. In this work, we study how training shapes the emergence of low-dimensional invariant manifolds in Reservoir Computing (RC) neural networks, a machine learning framework widely used for time-series prediction, system identification, and control \cite{Jaeger2004, Pathak2018, Griffith2019, Yan2024}. 
The original RC framework stems from Echo State Networks (ESNs) \cite{Jaeger2001} and Liquid State Machines (LSMs) \cite{Maass2002}.
At its core, RC combines the expressive dynamics of a recurrent neural network with a simple, efficient, and recursively implementable training step, which synthesizes a feedback controller for the RC system via least-squares optimization.
RC is also considered a compelling mechanistic model for brain computation \cite{maass2016searching,Bassett2020}, a domain where emergence of low-dimensional manifolds is thought to play a key role  \cite{Perich2025}. 
Our goal is to characterize the role and properties of such manifolds in RC using systems theory. 

To study how dominant manifolds are shaped by training, we consider 
a general linear reservoir in the infinite-data limit, in a time-series prediction task.
we show that the training data generate an invariant subspace of the trained reservoir, where the number of dominant modes is its dimension.
Next, we perform an explicit spectral analysis of a trained simplified diagonal linear RC system, and characterize its dominant eigenvalues and eigenvectors in terms of the training data, relating them to the spectral features of a backward  Dynamic Mode Decomposition (DMD) operator acting on the filtered input snapshots.
This shows a connection between trained RC dynamics and an
approximation of the backward Koopman evolution of
observables of the original dynamical system generating the training inputs. 
We then characterize these dominant modes through dominance theory \cite{forni2019differential}.
Finally, we discuss how these insights can potentially extend to nonlinear reservoirs through the lens of dominance theory, where differential analysis along trajectories provides a natural framework for studying the emergence of low-dimensional dominant invariant manifolds.

The paper is structured as follows. In Section \ref{sec:background} we review preliminaries on reservoir computing and dominance theory.  
In Section~\ref{sec:gen_lin_RC} we study dominance in general trained linear reservoirs.   
In Section \ref{sec:diag_lin_RC} we specialize to diagonal linear reservoirs, where we derive an explicit spectral characterization and establish the connection to Dynamic Mode Decomposition and the Koopman operator.  
In Section \ref{sec:nonlinear_RC} we discuss how these insights can generalize to the nonlinear setting through the lens of differential $p$-dominance. 
In Section \ref{sec:discussion} we summarize and discuss future directions.  


\section{Background} \label{sec:background}

\subsection{Notation and Preliminaries}

Vectors and matrices are denoted by lowercase and uppercase letters, respectively. The Euclidean space of dimension $n$ is denoted by $\mathbb{R}^n$, and $\| \cdot \|$ denotes the Euclidean norm for vectors and the Frobenius norm for matrices. We define the Kronecker delta as $\delta_{ij}=1$ if $i=j$, and $\delta_{ij}=0$ otherwise. The identity matrix is denoted by $I$, and $\mathbf{1}_n \in \mathbb{R}^n$ denotes the column vector with entries that are all equal to $1$. The transpose and conjugate transpose of a matrix $A$ are denoted by
$A^\top$ and $A^*$, respectively, and its Moore-Penrose pseudoinverse is denoted by $A^\dagger$.
We denote the rank, kernel, and image of a matrix $A$ by
$\operatorname{rank}(A)$, $\ker(A)$, and $\operatorname{Im}(A)$, respectively.
We write $\lambda_i[A]$ for the $i$th eigenvalue of $A$, 
and $\Re(\lambda)$ and $\Im(\lambda)$ for the real and imaginary parts, respectively, of a complex number $\lambda$.
For symmetric matrices $A$
, we write $A \succeq 0$ if $A$ is positive semidefinite, and $A \succ 0$ if $A$ is positive definite.
For a subspace $\mathcal{S}$, $\mathcal{S}^{\perp}$ denotes its orthogonal complement.
The symbol $\operatorname{diag}(x)$ denotes the diagonal matrix with diagonal entries given by the entries of the vector $x$, i.e.\ $[\operatorname{diag}(x)]_{ij} = x_i\,\delta_{ij}$.
The inertia of a symmetric matrix $P$ is denoted by $(p,0,n-p)$, meaning that $P$ has $p$ negative eigenvalues, $0$ zero eigenvalues, and $n-p$ positive eigenvalues.
For matrices $A$ and $B$, the nonzero eigenvalues of $AB$ and
$BA$ coincide, counting algebraic multiplicities.

\subsection{Reservoir Computing}
\label{sec:RC}

We study a continuous-time input affine reservoir computing model, with equations of the untrained reservoir reading
\begin{equation}
\begin{aligned}
\dot{r}(t) & =- \gamma r(t)+ \gamma f(Wr(t)+\sigma) + \gamma W_{\text{in}}u(t),
\end{aligned}
\label{eq:RC_before_training}
\end{equation}
where 
$r(t)\in\mathbb{R}^{n}$ is the reservoir state, $u(t)\in\mathbb{R}^{d}$ is the input, 
$\gamma>0$ is the timescale inverse constant,
$W_{\text{in}}\in\mathbb{R}^{n\times d}$ is the input matrix, 
$f:\mathbb{R}^{n}\to\mathbb{R}^{n}$ is a smooth nonlinear activation function acting entrywise on its vector of inputs,
$\sigma\in\mathbb{R}^{n}$ is a constant bias, and
$W\in\mathbb{R}^{n\times n}$ is a matrix of recurrent weights.

A reservoir is trained to forecast an input signal $u(t)$ observed over a time interval $t \in [0,T]$, typically stemming from observations of a dynamical system. During training,  \eqref{eq:RC_before_training} is continuously forced by the training input stream $u(t)$ and its response $r(t)$ is measured. During training the reservoir parameters $(W_{\text{in}},W, \gamma,\sigma)$ are fixed and initialized to satisfy the \textit{echo state property}
\cite{Jaeger2001}.  

The goal of reservoir training is to find a readout matrix ${W_{\text{out}}\in\mathbb{R}^{d\times n}}$ such that the output
\begin{equation}
    y(t) = W_{\text{out}}r(t) \label{eq:output}
\end{equation}
approximates the observed input as $y(t) \approx u(t)$ for ${t \in [0,T]}$.  
This is solved via a least squares minimization
\begin{equation}
\min_{W_{\text{out}}\in\mathbb{R}^{d\times n} }
\int_0^T \|W_{\text{out}} r(t) - u(t)\|^{2} dt 
\label{eq:min_W_out_continuous}
\end{equation}
\noindent 
whose minimum-norm solution is 
\begin{equation}
W_{\text{out}} =
C_{ur} C_{rr}^\dagger ,
\label{eq:Wout_continuous}
\end{equation}
where the uncentered second-moment matrices are
\begin{equation}
    C_{ur} = \frac{1}{T} \int_0^T u(t)r(t)^\top\,dt
    , \quad 
    C_{rr} = \frac{1}{T} \int_0^T r(t)r(t)^\top\,dt
    .
    \label{eq:Cur_Crr}
\end{equation}

For observations of $r(t)$ and $u(t)$ collected as discrete samples, this is solved via a least squares minimization
\begin{equation}
\min_{W_{\text{out}}\in\mathbb{R}^{d\times n} }\ \|W_{\text{out}} R - U\|^{2},
\label{eq:min_W_out_discrete}
\end{equation}
and the output map is therefore defined by the minimum-norm least-squares solution
\begin{equation}
W_{\text{out}} = U R^{\top} (R R^{\top})^{\dagger} = U R^\dagger,
\label{eq:W_out_discrete}
\end{equation}
where snapshots of the input and the reservoir states are collected at times $t_1,\dots,t_m \in [0,T]$ in the matrices
\begin{equation} \label{eq:input_data}
U=[u(t_1) \cdots u(t_{m})]
\in\mathbb{R}^{d\times m}
\end{equation}
and 
\begin{equation} \label{eq:reservoir_data}
R=[r(t_1) \cdots r(t_m)]
\in\mathbb{R}^{n\times m}.
\end{equation}
In this paper, we assume that the training observations \eqref{eq:input_data},\eqref{eq:reservoir_data} are made at regular time intervals $t_{j+1} - t_j := h $.

To forecast the time-series $u(t)$ for $t > T$, we close the loop of the RC using the trained output \eqref{eq:output},
yielding the autonomous system   
\begin{equation}
\begin{aligned}
\dot{r}(t) &= - \gamma r(t)+ \gamma f(Wr(t) +\sigma)+ \gamma W_{\text{in}}W_{\text{out}} r(t),
\end{aligned}
\label{eq:RC_after_training}
\end{equation}
which is simulated forward in time for $t > T$.

\subsection{$p$-Dominance}
\label{sec:p-dominance}

$p$-Dominance \cite{forni2019differential} 
characterizes systems with
asymptotic behavior that is governed by a low-dimensional dominant subspace, for which
exactly $p$ directions to dominate the long-term dynamics, while the remaining
$n-p$ directions are uniformly contracting.

\begin{definition}[$p$-dominance]
\label{def:p_dom}
A linear system $\dot{x}=Ax$ is $p$-dominant with rate $\lambda\ge 0$ if there exists a symmetric matrix $P$ with inertia $(p,0,n-p)$ such that
\begin{equation}
A^\top P + P A \le -2\lambda P - \varepsilon I
\end{equation}
\noindent 
for some $\varepsilon\ge 0$; the property is strict if $\varepsilon>0$.
\end{definition}\noindent
For $\varepsilon>0$, such a $P$ exists if and only if $A+\lambda I$ has exactly $p$ eigenvalues with positive real part and $n-p$ eigenvalues with negative real part, counting multiplicity \cite[Prop.~1]{forni2019differential}. 

\begin{definition}[differential $p$-dominance]
\label{def:diff_p_dom}
A nonlinear autonomous system
$\dot{x} = f(x)$ with $x \in \mathbb{R}^n$
and smooth $f$ is (differentially) $p$-dominant with rate $\lambda \ge 0$
if there exists a constant storage matrix $P = P^\top$ with inertia $(p,0,n-p)$ such that
the prolonged system
\begin{equation}
\begin{aligned}
\dot{x} &= f(x),
\qquad
\delta \dot{x} = \partial f(x)\,\delta x 
\end{aligned}
\label{eq:prolonged_system}
\end{equation}
satisfies the conic constraint
\begin{equation}
\begin{bmatrix}
\delta \dot{x} \\
\delta x
\end{bmatrix}^\top \!
\begin{bmatrix}
0 & P \\
P & 2\lambda P + \varepsilon I
\end{bmatrix} \!
\begin{bmatrix}
\delta \dot{x} \\
\delta x
\end{bmatrix}
\le 0, 
 \ \forall\, (x,\delta x) \in \mathbb{R}^n\times\mathbb{R}^n, 
\label{eq:p_dominance}
\end{equation}
for some $\varepsilon > 0$.
\end{definition}

From~\eqref{eq:prolonged_system} and the differential conic constraint \eqref{eq:p_dominance},
$p$-dominance requires that
\begin{equation}
\partial f(x)^\top P + P\,\partial f(x) \;\leq\; -2\lambda P - \varepsilon I,
\qquad \forall\, x \in \mathbb{R}^n .
\label{eq:p_dominance_LMI}
\end{equation}
\noindent 
Intuitively, this condition enforces exponential contraction, at rate at least $\lambda$, of all infinitesimal displacements transversal to a $p$-dimensional dominant subspace.
For strictly $p$-dominant systems with $p\in\{0,1,2\}$, every bounded trajectory converges to a simple attractor \cite[Cor.~1]{forni2019differential}: for ${p=0}$, to a unique equilibrium; for $p=1$, to an equilibrium (possibly one of several, i.e., multistability); for ${p=2}$, to an equilibrium, a set of equilibria and connecting arcs, or a limit cycle. 


\section{Dominant Modes in Trained Linear Reservoirs}
\label{sec:gen_lin_RC}


First we consider a linear RC model obtained by specializing \eqref{eq:RC_before_training} to a linear activation function $f(x) = x$, and choosing parameters $\gamma = 1$ and $\sigma = \sigma_b \mathbf{1}_n$, yielding
\begin{equation}
\begin{aligned}
\dot{r}(t) &= (- I + W) r(t) + W_{\mathrm{in}} u(t) + \sigma_b \mathbf{1}_n,
\label{eq:RC_linear_untrained}
\end{aligned}
\end{equation}
where $\sigma_b \in \mathbb{R}$, and $\mathbf{1}_n \in \mathbb{R}^n$ is the vector of ones.
After training, the closed-loop reservoir is a linear, time-invariant dynamical system
\begin{equation}
    \dot{r}(t) = \big(-I + W + W_{\mathrm{in}}W_{\mathrm{out}}\big) r(t) + \sigma_b \mathbf{1}_n. \label{eq:RC_linear_trained}
\end{equation}
\noindent 
Let $J_0 := -I + W$ and $J_T := J_0 + W_{\mathrm{in}} W_{\mathrm{out}}$ denote the state matrices of the untrained reservoir \eqref{eq:RC_linear_untrained} and of the trained reservoir \eqref{eq:RC_linear_trained}, respectively.  $J_0$ is chosen Hurwitz, such that the untrained reservoir is globally exponentially stable in the absence of driving input.  Equivalently, the untrained reservoir is strictly $0$-dominant with rate $\lambda=0$. 

In this section we show that, in the infinite-data limit, training turns the $0$-dominant untrained reservoir into a strictly $q$-dominant one, where $0\leq q\leq n$ is determined by the training data. 
First, we prove a useful uncentered second-moment identity for the linear RC that reveals an invariant subspace generated by the training data. We then use this to characterize the dominant eigenvalues of the trained reservoir.

\begin{lemma}[Second-moment identity]
    Consider the driven linear reservoir \eqref{eq:RC_linear_untrained} and its trained counterpart \eqref{eq:RC_linear_trained}, with $\sigma_b=0$. 
    Define $r_u(t)$ as the solution of \eqref{eq:RC_linear_untrained} for $t \in [0,T]$.
    Suppose the output matrix is given by \eqref{eq:Wout_continuous}, where
    \begin{equation}
    \begin{aligned}
        C_{rr} &= \lim_{T \to \infty} 
        \frac{1}{T}
        \int_0^T r_u(t) r_u(t)^{\top} dt, 
        \\
        C_{ur} &= \lim_{T\to \infty} 
        \frac{1}{T}
        \int_{0}^T u(t) r_u(t)^{\top} dt, \\ \label{eq:Cur_Crr_limit}
    \end{aligned}
    \end{equation}%
    assuming 
    that the limits in \eqref{eq:Cur_Crr_limit} exist. 
    Then 
    \begin{equation}
        J_T C_{rr} + C_{rr} J_T^{\top} = 0
        \label{eq:JT_Sylvester}
    \end{equation}
    and $\operatorname{Im}C_{rr}$ is an invariant subspace of $J_T$ 
    of dimension $\operatorname{rank}(C_{rr})$. 
\label{lem:covariance_identity}
\end{lemma}

\begin{proof}
Since $r_u$ is bounded, we have
\begin{equation}
\begin{aligned}
0
&=
\lim_{T\to\infty}\frac{1}{T}
\int_0^T
\frac{d}{dt}\left(r_u r_u^{\top}\right)dt 
\\
&=
J_0C_{rr}+W_{\mathrm{in}}C_{ur}
+C_{rr}J_0^{\top}+C_{ur}^{\top}W_{\mathrm{in}}^{\top}.
\end{aligned}
\label{eq:covariance_intermediate}
\end{equation}
We now show that $\ker C_{rr}\subseteq\ker C_{ur}$ and find an expression for $C_{ur}$. Let $v\in\ker C_{rr}$. Then
\begin{equation}
0
=
v^{\top}C_{rr}v
=
\lim_{T\to\infty}
\frac1T\int_0^T |r_u(t)^{\top}v|^2dt.
\end{equation}
Using the triangle inequality and then Cauchy-Schwarz inequality, we have
\begin{equation}
\begin{aligned}
\left\| C_{ur}v \right\|
& 
=
\left\|
\frac1T\int_0^T
u(t)r_u(t)^{\top}v\,dt
\right\| \\
& \leq
\frac1T\int_0^T
\|u(t)\||r_u(t)^{\top}v|\,dt \nonumber
\\
&
\leq
\left(
\frac1T\int_0^T\|u(t)\|^2dt
\right)^{1/2} \! 
\left(
\frac1T\int_0^T|r_u(t)^{\top}v|^2dt
\right)^{1/2} \! .
\end{aligned}
\end{equation}
Since $u(t)$ is bounded, the right-hand side converges to zero, and thus $C_{ur}v=0$.
By the properties of the Moore-Penrose inverse, $C_{rr}^{\dagger}C_{rr}$ is the orthogonal projector onto $(\ker C_{rr})^{\perp}$, the set of vectors orthogonal to every vector in $\ker C_{rr}$. 
Since $\ker C_{rr}\subseteq\ker C_{ur}$, it follows that
$C_{ur}=C_{ur}C_{rr}^{\dagger}C_{rr}$. Therefore, we have
\begin{equation}
C_{ur}=W_{out}C_{rr},
\qquad
C_{ur}^{\top}=C_{rr}W_{out}^{\top}.
\end{equation}
Substituting into \eqref{eq:covariance_intermediate} gives \eqref{eq:JT_Sylvester}.

We notice that \eqref{eq:JT_Sylvester} has the form of the homogeneous Sylvester equation.
Let $x\in\operatorname{Im}(C_{rr})$, then there exists $v$ such that $x=C_{rr}v$, and hence
\begin{equation}
J_Tx
=
J_TC_{rr}v
=
-C_{rr}J_T^{\top}v
\in\operatorname{Im}(C_{rr}).
\end{equation}
Thus 
$J_T\operatorname{Im}(C_{rr})
\subseteq
\operatorname{Im}(C_{rr})$,
and it follows that $\operatorname{Im}(C_{rr})$ is an invariant subspace of $J_T$, with dimension $\operatorname{rank}(C_{rr})$.
\end{proof}


\begin{assumption}
\label{ass:symmetrized_J0} 
There 
exists $\lambda>0$ such that for every $x\in \operatorname{Im}(C_{rr})^{\perp}\setminus\{0\}$,
\begin{equation}
x^{\top}
\frac{J_0+J_0^{\top}}{2}
x
<
-\lambda \|x\|^2.
\end{equation}
\end{assumption}


\begin{theorem}[Generation of a dominant subspace after training]
\label{thm:lin_res_dom_subspace}
Consider the trained reservoir \eqref{eq:RC_linear_trained}. 
Let the hypotheses of Lemma~\ref{lem:covariance_identity} and Assumption~\ref{ass:symmetrized_J0} hold, and let $q=\operatorname{rank}(C_{rr})\le n$.
Then, exactly $q$ eigenvalues of $J_T$, counting algebraic multiplicity, are dominant with rate $\lambda$, and satisfy
\begin{equation}
\operatorname{Re}(\lambda_{\rm dom}[J_T]) = 0,  
\label{eq:q-dom_Evals_JT}
\end{equation}
i.e. the $n - q$ non-dominant eigenvalues satisfy $\operatorname{Re}(\lambda_{\rm nd}[J_T]) < - \lambda$.
\end{theorem}

\begin{proof}
If $q=0$, then $C_{rr}=0$, so $W_{\mathrm{out}}=0$ and $J_T=J_0$.
By Assumption \ref{ass:symmetrized_J0}, all eigenvalues of $J_T$ satisfy
$\operatorname{Re}(\lambda_j[J_T])<-\lambda$, so \eqref{eq:q-dom_Evals_JT} immediately follows.

Next, we consider $q>0$.
Let
\begin{equation*}
\mathcal{S}=\operatorname{Im}(C_{rr}),
\qquad
\mathcal{S}^{\perp}=\ker(C_{rr}).
\end{equation*}
Since $C_{rr}$ is symmetric, these subspaces are orthogonal complements.
Let $V_{\mathcal{S}}\in\mathbb{R}^{n\times q}$ and
$V_{\mathcal{S}^{\perp}}\in\mathbb{R}^{n\times(n-q)}$ have orthonormal columns spanning
$\mathcal{S}$ and $\mathcal{S}^{\perp}$, respectively.
By Lemma \ref{lem:covariance_identity},
$\mathcal{S}$ is invariant under $J_T$. 
Thus, by \cite[Theorem~9.3.1]{nicholson2023linear}, in the basis
$[V_{\mathcal{S}}\;V_{\mathcal{S}^{\perp}}]$, the matrix $J_T$ has the block upper triangular form
\begin{equation*}
\tilde J_T
=
\begin{pmatrix}
J_1 & J_2\\
0 & J_4
\end{pmatrix},
\end{equation*}
where $J_1=V_{\mathcal{S}}^{\top}J_TV_{\mathcal{S}} \in\mathbb{R}^{q\times q}$ represents the restriction
$J_T|_{\mathcal{S}}$, 
$J_2 = V_{\mathcal{S}}^{\top}J_TV_{\mathcal{S}^{\perp}} \in\mathbb{R}^{q\times(n-q)}$, 
and
$J_4 = V_{\mathcal{S}^{\perp}}^{\top}J_TV_{\mathcal{S}^{\perp}} \in\mathbb{R}^{(n-q)\times(n-q)}$.
In the same basis,
\begin{equation}
\tilde C_{rr}
=
\begin{pmatrix}
C & 0\\
0 & 0
\end{pmatrix},
\qquad
C=V_{\mathcal{S}}^{\top}C_{rr}V_{\mathcal{S}}.
\end{equation}
Since $C_{rr}$ is symmetric and positive semidefinite, and $V_{\mathcal{S}}$ spans $\operatorname{Im}(C_{rr})$,
$C$ is symmetric and positive definite.

Since $\tilde J_T$ is block upper triangular, by \cite[Theorem~9.3.2]{nicholson2023linear}, the spectrum of $J_T$ is the union of the spectra of the two diagonal blocks,
\begin{equation}
\lambda[J_T]=\lambda[J_1]\cup\lambda[J_4].
\label{eq:JT_Evals}
\end{equation}
We now show that (1) $J_1$ has $q$ purely imaginary eigenvalues,
and (2) all the remaining $n-q$ eigenvalues of $J_4$ are in the left-half-plane.

First, let 
\begin{equation*}
\tilde J_1 = C^{-1/2} J_1 C^{1/2}.
\end{equation*}
Following from Lemma \ref{lem:covariance_identity},
we have $J_1 C + C J_1^{\top}=0$.
Since $C$ is symmetric and positive definite, multiplying by
$C^{-1/2}$ on both sides gives
$\tilde J_1+\tilde J_1^{\top}=0$.
Thus, $\tilde J_1$ is skew-symmetric, and since it is similar to
$J_1$, all $q$ eigenvalues of $J_1$ are purely imaginary, so
\begin{equation}
\operatorname{Re}(\lambda_j[J_1])=0
, \qquad \forall j.
\label{eq:J1_Evals}
\end{equation}

We now consider $J_4$. 
Since $C_{rr}$ is real and symmetric, and using the standard Moore-Penrose identity, 
we have
$\ker(C_{rr}^{\dagger})
=
\ker(C_{rr}^{*})
=
\ker(C_{rr})$.
Since 
the columns of $V_{\mathcal{S}^{\perp}}$ span
$\ker(C_{rr})$, it follows that
$W_{\mathrm{out}}V_{\mathcal{S}^{\perp}}
=
C_{ur}C_{rr}^{\dagger}V_{\mathcal{S}^{\perp}}
=
0$.
Therefore,
\begin{equation*}
J_4
=
V_{\mathcal{S}^{\perp}}^{\top}J_TV_{\mathcal{S}^{\perp}}
=
V_{\mathcal{S}^{\perp}}^{\top}J_0V_{\mathcal{S}^{\perp}}.
\end{equation*}
Since the columns of $V_{\mathcal{S}^{\perp}}$ form an orthonormal basis of
$\mathcal{S}^{\perp}$, 
Assumption \ref{ass:symmetrized_J0} gives 
\begin{equation*}
\frac{J_4+J_4^{\top}}{2}
=
V_{\mathcal{S}^{\perp}}^{\top}
\frac{J_0+J_0^{\top}}{2}
V_{\mathcal{S}^{\perp}}
\prec
-\lambda I.
\end{equation*}
Let $\mu$ be an eigenvalue of $J_4$ with eigenvector $z\neq 0$. Since $J_4$ is real,
\begin{equation*}
\operatorname{Re}(\mu)
=
\frac{
z^{*}
\left(
\frac{J_4+J_4^{\top}}{2}
\right)
z
}{
z^{*}z
}
<
-\lambda.
\end{equation*}
Thus, all $n-q$ eigenvalues of $J_4$ satisfy
\begin{equation}
\operatorname{Re}(\lambda_j[J_4])<-\lambda
, \qquad \forall j.
\label{eq:J4_Evals}
\end{equation}

Finally, using \eqref{eq:JT_Evals} with \eqref{eq:J1_Evals}, \eqref{eq:J4_Evals}
we get that $J_T$ has exactly $q$ eigenvalues, counting algebraic multiplicity, satisfying
$\operatorname{Re}(\lambda[J_T]) = 0$, and $n-q$ eigenvalues satisfing
$\operatorname{Re}(\lambda_j[J_4])<-\lambda$. Hence, \eqref{eq:q-dom_Evals_JT} follows.

\end{proof}

 
For the exact infinite-time result, we require Assumption \ref{ass:symmetrized_J0} on
$\operatorname{Im}(C_{rr})^\perp$. In standard RC, training is performed over a finite time interval, so this result is only approximated. In practice, $J_0$ is usually chosen Hurwitz so that the untrained reservoir is stable and has the fading memory property, which is typically sufficient for practical RC operation, but does not in general imply Assumption \ref{ass:symmetrized_J0}. For the choice $J_0=-I+\omega I$ with $\omega<1$, Assumption \ref{ass:symmetrized_J0} is automatically satisfied for any $0<\lambda<1-\omega$.

Moreover, under Assumption~\ref{ass:symmetrized_J0}, when $J_0$ is Hurwitz, we can choose $\lambda>0$ sufficiently small such that
\begin{equation}
\max_{\mu\in\sigma(J_0)}
\operatorname{Re}(\mu)
<
-\lambda.
\end{equation}
Thus, the untrained reservoir has no dominant eigenvalues with respect to $\lambda$, while $J_T$ has exactly $q$ such eigenvalues. Hence, the $q$ dominant eigenvalues are directly generated by training, and depend on the input data.


Several previous works have considered correlations between recurrent connectivity and low-rank feedback \cite{schuessler2020dynamics}, in particular in the training of reservoirs with periodic input time series \cite{susman2021quality,HakimKarma2026}. 
Most recently, 
\cite{HakimKarma2026} 
studies the emergence of ``slow'', here called dominant, modes through reservoir training.
For linear reservoirs trained on periodic time series with a finite
number of Fourier modes, this work shows that training generates a corresponding finite set of modes associated with the frequencies present in the data.
Our Theorem \ref{thm:lin_res_dom_subspace} is a complementary and more general characterization of this phenomenon for linear reservoirs, as it  only assumes boundedness of signals and the existence of
uncentered second-moment matrices in \eqref{eq:Cur_Crr_limit}, rather than explicit periodicity in the cited works. 
Theorem \ref{thm:lin_res_dom_subspace} also does not rely on random matrix constructions for the reservoir connectivity as \cite{schuessler2020dynamics,susman2021quality,HakimKarma2026}, instead making only a general assumption of contractivity on the unexcited reservoir directions. Finally, our framework allows arbitrarily multi-dimensional inputs while \cite{susman2021quality,HakimKarma2026} assume a linear feedback of rank $1$.


\section{Dominant Diagonal Reservoirs}
\label{sec:diag_lin_RC}

We now consider the linear reservoir as an even more simplified model, where the effect of training on the reservoir dynamics can be analyzed explicitly. 
We take in \eqref{eq:RC_linear_untrained} and \eqref{eq:RC_linear_trained}
$W = \omega I$, 
where $\omega \in \mathbb{R}$ satisfies $\omega<1$.
Here, by design, $J_0 = (-1 + \omega)I$ is Hurwitz, and all eigenvalues of $J_0$ coincide. By constructing an explicit spectral solution to the trained diagonal reservoir, we will connect its dominant modes established by Theorem \ref{thm:lin_res_dom_subspace} to DMD modes of the filtered inputs during training.
%

\subsection{Spectrum Analysis of Diagonal Trained Reservoir}
\label{subsec:spectrum_analysis}

Here we show that the trained diagonal linear reservoir \eqref{eq:RC_linear_trained} possesses dominant modes that are directly characterized by properties of
the training data $U$.
To do this, we derive expressions for the eigenvalues and eigenvectors of $J_T$ and leverage explicit solutions of LTI systems. 
First, we will prove a useful lemma relating the product of the trained output matrix $W_{\text{out}}$ with the input matrix $W_{\text{in}}$ to the output time series $U$. This quantity will be useful for describing the trajectories of trained reservoirs. To state the lemma, we first introduce some notation. Observe that from the variation of constants formula with $r(0) = 0$, each data snapshot of the open-loop reservoir state $r(t)$ of \eqref{eq:RC_linear_untrained} at time $t_i$ during training can be expressed explicitly as 
\begin{equation}
    \small r(t_i) = W_{\mathrm{in}}\int_0^{t_i} e^{(\omega-1)(t_i-\tau)}u(\tau)\,d\tau + \frac{\sigma_b\left( e^{(\omega - 1) t_i} - 1 \right)}{\omega - 1}  \mathbf{1}_n. \label{eq:reservoir_VoC}
\end{equation}
\noindent 
Accordingly, we write $r(t_i) = W_{\mathrm{in}} b_1(t_i)+ b_2(t_i) \mathbf{1}_n$ with $b_{1}(t_i) := \int_0^{t_i} e^{(\omega-1)(t_i-\tau)}u(\tau)\,d\tau\in \mathbb{R}^d$ and $b_2(t_i) := \sigma_b\big(e^{(\omega-1)t_i}-1\big)/(\omega-1)\in \mathbb{R}$, and define the matrices  
\begin{equation} \small
B_1 := [\,b_1(t_1)\ \cdots\ b_1(t_m)\,] ,\ 
B_2 := [\,b_2(t_1)\ \cdots\ b_2(t_m)\,]. \label{eq:B_matrices}
\end{equation}
Next, we make a standing data informativity assumption.
\begin{assumption}\label{ass:data_inform} Consider \eqref{eq:RC_linear_untrained} over a training interval $t \in [0,T]$ with $\omega < 1$.  
We make the following assumptions: 

1) $u \in L^2([0,T]; \mathbb{R}^d)$;

2) If $\sigma_b = 0$, $\operatorname{rank}(B_1) = d$; 

3) 
If $\sigma_b \neq 0$, $\operatorname{rank} \begin{bmatrix} B_1^{\top} & B_2^{\top} \end{bmatrix}^{\top} = d+1$. 

\end{assumption}

\begin{lemma} \label{lem:WoutWin}
Consider 
\eqref{eq:RC_linear_untrained} under Assumption \ref{ass:data_inform}, and 
\eqref{eq:RC_linear_trained}, with $\mathrm{W_{out}}$ defined through \eqref{eq:W_out_discrete} with $r(0) = 0$. Assume $\omega < 1$,
$\mathrm{W_{in}}$ is full column rank $d$, $\mathbf{1}_{n}\not \in \operatorname{Im} (\mathrm{W_{in}})$.
Then for $\sigma_b \neq 0$,
\begin{equation}
     \mathrm{W_{out}} \mathrm{W_{in}} = U \Big(I_{m} - s^{-1} \big(I_m - B_1^{\dagger} B_1\big) B_2^T B_2 \Big) B_1^{\dagger} \label{eq:WoutWin_lem}
\end{equation}
where $B_1, B_2$ are as in \eqref{eq:B_matrices} and $s$ is defined as
\begin{equation}
    s = B_2 \big(I - B_1^{\dagger} B_1\big) B_2^T \in \mathbb{R}.
\end{equation}
If $\sigma_b = 0$ then instead, 
\begin{equation}
    \mathrm{W_{out}} \mathrm{W_{in}} = U B_1^{\dagger}. \label{eq:WoutWin_lem_nobias}
\end{equation}
\end{lemma}
\begin{proof} Observe that the reservoir snapshot matrix \eqref{eq:reservoir_data} is
    \begin{equation}
R=W_{\mathrm{in}}B_1+\mathbf{1}_n B_2
=
\bigl[\,W_{\mathrm{in}}\ \mathbf{1}_n\,\bigr]
\begin{bmatrix}
B_1\\
B_2
\end{bmatrix}
:=
Z\tilde B.
\end{equation}
Since $W_{\mathrm{in}}$ has full column rank and $\mathbf{1}_n$ is linearly independent of its columns, $Z$ has full column rank. 
Thus, every matrix
$W_{\rm out}Z \in\mathbb R^{d\times(d+1)}$ can be realized by some
$W_{\rm out}$, and the least-squares training step \eqref{eq:min_W_out_discrete} can be rewritten as 
\begin{equation*}
\min_{M \in\mathbb R^{d\times(d+1)}}
\|M \tilde{B}-U\|^2 ,\qquad M=W_{\rm out}Z .
\end{equation*}
Since $\tilde B$ has full row rank by Assumption~\ref{ass:data_inform}, this problem has the unique minimizer $M=U\tilde B^{\dagger}=U\tilde B^{\top}(\tilde B\tilde B^{\top})^{-1}$. Hence, every least-squares solution $W_{\rm out}$ of \eqref{eq:min_W_out_discrete} satisfies $W_{\rm out}Z=U\tilde B^{\dagger}$. 
$W_{\rm in}
=
Z
\begin{bmatrix}
I_d\\
0
\end{bmatrix}$,
we have
\begin{equation}
W_{\rm out}W_{\rm in}
=
W_{\rm out} Z
\begin{bmatrix}
I_d\\
0
\end{bmatrix}
=
U \tilde{B}^\dagger
\begin{bmatrix}
I_d\\
0
\end{bmatrix}. \label{eq:WoutWin_intermediate_2}
\end{equation}
%
%
%
Then the final expression \eqref{eq:WoutWin_lem} follows from \eqref{eq:WoutWin_intermediate_2} by explicit computation of $\tilde B^{\dagger} = \tilde B^T (\tilde B \tilde B^T)^{-1}$ using the Schur complement formula for the inverse of $\tilde B \tilde B^T = \begin{bmatrix} B_1 B_1^T & B_1 B_2^T \\ B_2 B_1^T & B_2 B_2^T \end{bmatrix}$, leveraging the fact that $B_1$ is full row rank and therefore $B_1 B_1^T$ is invertible. 
For $\sigma_b = 0$, $R = W_{\mathrm{in}}B_1$ and the same argument with $Z=W_{\mathrm{in}}$, $\tilde B=B_1$ gives \eqref{eq:WoutWin_lem_nobias}. 
\end{proof}

\noindent With this Lemma established, we state and prove the main result of this section.

\begin{theorem}[Trained Reservoir Trajectories]
\label{thm:trained_reservoir_trajectories}
Consider
\eqref{eq:RC_linear_untrained} under Assumption \ref{ass:data_inform}, and 
\eqref{eq:RC_linear_trained} with $W_{out}$ obtained via the training step \eqref{eq:W_out_discrete} with $r(0) = 0$. Assume $d<n$, $\omega < 1$, 
$W_{in}$ is full column rank $d$, $\mathbf{1}_{n}\not \in \operatorname{Im} (W_{\mathrm{in}})$,  $U$ has full row rank, 
and $J_T= (-1 + \omega)I + W_{\mathrm{in}} W_{\mathrm{out}}$ is diagonalizable and invertible. The following hold:

1) The trajectories of \eqref{eq:RC_linear_trained} are given by 
\begin{equation}
    r(t) = - J_T^{-1}\sigma_b \mathbf{1}_n + \sum_{i = 1}^n c_i e^{\lambda_{i}[J_T] t} v_i, \label{eq:reservoir_traj}
\end{equation}
and $v_i$ is a right eigenvector of $J_T$ corresponding to $\lambda_i[J_T]$;

2) $J_T$ has at least $n - d$ unperturbed eigenvalues $\lambda_k[J_T] = -1 + \omega$, with $k = 1, \dots, n_u$, $n_u \geq n-d$ and at most $d$ shifted eigenvalues $\lambda_i[J_T] = -1 + \omega + \lambda_{i}\left[\mathrm{W_{out}} \mathrm{W_{in}}\right] \neq -1 + \omega$ where $i\in \{1,...,n_s\}$ with $n_s = \operatorname{rank}( \mathrm{W_{out}} \mathrm{W_{in}}) \leq d$, $n_u + n_s = n$, and $\mathrm{W_{out}} \mathrm{W_{in}}$ is defined by \eqref{eq:WoutWin_lem} for $\sigma_b \neq 0$ and by \eqref{eq:WoutWin_lem_nobias} for $\sigma_b = 0$. If $y_i$ is an eigenvector of $\mathrm{W_{out}} \mathrm{W_{in}}$ corresponding to a nonzero eigenvalue $\lambda_i[\mathrm{W_{out}} \mathrm{W_{in}}]$, then $v_i =  \mathrm{W_{in}} y_i$ is an eigenvector of $J_T$ corresponding to $\lambda_i[J_T]$.

3) Consider $\sigma_b = 0$ and let the training signal $u(t)$ be piecewise constant on each interval, i.e. $u(t) = u(t_k)$ for  $t \in (t_{k-1}, t_{k}]$ for $k = 2, \dots, m$ where $t_0 = 0$ and $t_k = k h$ for some time step $h > 0$. Define the unit right-shift matrix $S\in\mathbb{R}^{m\times m}$, such that 
$S_{jk}=\delta_{j,(k+1)}$, and $\alpha = e^{-(1-\omega) h}$. Then the shifted eigenvalues of $J_T$ are
\begin{equation}
    \lambda_i[J_T] = -1 + \omega + \frac{1 - \omega}{1 - \alpha} \lambda_i\left[U \left( U \big(I - \alpha S^\top\big)^{-1} \right)^{\dagger}\right].
    \label{eq:lin_explicit_Evals_gen}
\end{equation}
For every right eigenvector $y_i$ of  $U \left( U \big(I - \alpha S^\top\big)^{-1} \right)^{\dagger}$ corresponding to a nonzero eigenvalue, $J_T$ has an eigenvector of $\lambda_i[J_T]$ defined by $v_i = W_{in} y_i$. 
\end{theorem}

\begin{proof}
1) is an application of standard solutions to linear time invariant systems \cite{chen1984linear}. To prove 2), we observe that $\lambda_i[J_T] = -1 + \omega + \lambda_i[ W_{\text{in}} W_{\text{out}}]$, and the corresponding eigenvectors $v_i$ coincide with those of $W_{\text{in}}W_{\text{out}}$. Dimensionality of the null space and the perturbed eigenspace follow from the observation that $W_{\text{in}} \in \mathbb{R}^{n \times d}$ and $W_{\text{out}}\in \mathbb{R}^{d \times n}$ with $d < n$, therefore the low-rank perturbation  $W_{\text{in}} W_{\text{out}}$ has at most $d$ nonzero eigenvalues.
Furthermore, recall that for nonzero eigenvalues of $W_{\mathrm{in}}W_{\mathrm{out}}$, $\lambda_i[W_{\text{in}}W_{\text{out}}] = \lambda_i[W_{\text{out}}W_{\text{in}}]$ and $v_i = W_{\text{in}} y_i$, where $y_i \in \mathbb{R}^d$ is a right eigenvector of $W_{\text{out}}W_{\text{in}}$ corresponding to $\lambda_i[W_{\text{out}}W_{\text{in}}]$. Then to characterize the shifted eigenvalues $\lambda_i[J_T]$ and the eigenvectors $v_i$ we must simply characterize $\lambda_i[W_{\text{out}}W_{\text{in}}]$ and $y_i$. The theorem statement follows from the expression for $W_{\rm out}W_{\rm in}$ derived in Lemma \ref{lem:WoutWin}.

For part 3), we do this explicitly for the case $\sigma_b = 0$ and piecewise-constant inputs.
Suppose that $u(t)=u(t_{k})$ for $t\in(t_{k-1},t_k]$, then, 
\begin{equation} \label{eq:b1}
\begin{aligned}
b_1(t_j)
& =
\sum_{k=1}^j \int_{t_{k-1}}^{t_k} e^{(\omega-1)(t_j-\tau)}u(\tau)\,d\tau
\\
& =
\sum_{k=1}^j u_k \int_{t_{k-1}}^{t_k} e^{(\omega-1)(t_j-\tau)}\,d\tau
= \frac{1-\alpha}{1-\omega} \sum_{k=1}^j \alpha^{j-k}u(t_k).
\end{aligned}
\end{equation}

\noindent 
Then $B_1 = U K^\top$ where $K \in \mathbb{R}^{m \times m}$ is the lower triangular matrix with entries $K_{jk} = \frac{1-\alpha}{1-\omega}\alpha^{j-k}$, $k \le j$, and $K_{jk}=0$ otherwise, alternatively written as
\begin{equation}
K = \frac{1-\alpha}{1-\omega}\sum_{q=0}^{m-1}\alpha^q S^q
=
\frac{1-\alpha}{1-\omega}(I-\alpha S)^{-1},
\end{equation}
since $S^q=0$ for all $q\ge m$, and $K$ is invertible. 
Thus, from Lemma \ref{lem:WoutWin},
\begin{equation}
W_{\mathrm{out}}W_{\mathrm{in}}
= \frac{1-\omega}{1-\alpha} \, U \left( U \big(I - \alpha S^\top\big)^{-1} \right)^\dagger,
\end{equation}
and \eqref{eq:lin_explicit_Evals_gen} directly follows.
\end{proof}

\subsection{Connection to Dynamic Mode Decomposition}

For $\sigma_b = 0$, we showed in Lemma \ref{lem:WoutWin} that the effective feedback matrix obtained through training is $W_{\rm out}W_{ \rm in} = U B_1^{\dagger}$ where entries of $B_1$ are filtered inputs $b_1(t_j)$ defined in \eqref{eq:b1}.  We can relate this characterization to the classical Dynamic Mode Decomposition (DMD) algorithm, which fits a linear evolution map to observed data and characterizes its eigenvalues and modes \cite{tu2014DMD}. A potential connection between DMD and reservoir computing was previously suggested in \cite{bollt2021explaining}; we make the connection explicit for diagonal reservoirs trained with piecewise constant inputs as in Theorem \ref{thm:trained_reservoir_trajectories}.3. To do this, we apply the DMD construction to the filtered input sequence $b_1(t_j)$, fitting a map that predicts $b_1(t_{j-1})$ from $b_1(t_j)$. This backwards-in-time regression is called the \textit{backwards} DMD step, and has appeared in forward-backwards DMD methods in the literature, typically to reduce noise effects on the operator estimation \cite{dawson2016characterizing,lortie2026forward}. In the following Corollary we connect the spectrum of a trained diagonal reservoir to that of the operator learned via backwards DMD.

\begin{corollary} [Backward DMD representation]
\label{cor:DMD}    Consider \eqref{eq:RC_linear_untrained}, \eqref{eq:RC_linear_trained} under the assumptions of Theorem \ref{thm:trained_reservoir_trajectories}, with $\sigma_b=0$ and $J_T = (-1+\omega)I+W_{\rm in} W_{\rm out}$. Define 
    \begin{equation*}
        B_{-} = \begin{bmatrix}
            b_1(t_0) & \dots & b_1(t_{m-1})
        \end{bmatrix}, \quad b_1(t_0) = 0.
    \end{equation*}
    The backward DMD matrix $A_b = B_{-} B_1^{\dagger}$ is the unique minimizer of $\| A B_1 - B_{-}\|$ over $A \in \mathbb{R}^{d\times d}$. It satisfies 
    \begin{equation} \label{eq:JTWin}
        J_T W_{\rm in} = \frac{(1-\omega) \alpha}{1 - \alpha} W_{\rm in} (I - A_b).
    \end{equation}
    Then each of the $d$ right eigenpairs $A_b y_i = \mu_i y_i$ induces a right eigenpair $\lambda_i[J_T]$, $v_i$ for the trained reservoir as
    \begin{equation}
        \lambda_i(J_T) = \frac{(1 - \omega) \alpha}{1 - \alpha} (1 - \mu_i), \quad v_i = W_{\rm in} y_i. \label{eq:eigenpairs}
    \end{equation}
    The remaining $n-d$ eigenvalues of $J_T$ equal $-1 + \omega$.
\end{corollary}
\begin{proof}
    Observe that we can equivalently write \eqref{eq:b1} as 
    \begin{equation}
        U = \frac{1 - \omega}{1 - \alpha} (B_1 - \alpha B_-). 
    \end{equation}
    Then from Lemma \ref{lem:WoutWin} it follows that 
    \begin{equation}
        W_{\rm out} W_{\rm in} = U B_1^{\dagger} = \frac{1 - \omega}{1 - \alpha}  (I_d - \alpha A_b)
    \end{equation}
    where we used $B_1 B_1^{\dagger} = I_d$ and the definition of $A_b$. Substituting this expression into $J_T W_{\rm in}$ using the definition  
    of $J_T$ 
    we arrive at \eqref{eq:JTWin}, from which the eigenpair relationship \eqref{eq:eigenpairs} follows directly by acting on $y_i$ with $J_T W_{\rm in}$. 
\end{proof}
Corollary \ref{cor:DMD} describes how the dominant modes of a trained reservoir are generated by the eigenvalues and 
eigenmodes 
of the backwards DMD operator fitted to the filtered input history. This connection establishes a concrete link between reservoir training and Koopman theory \cite{otto2021koopman}, which describes how observables, i.e. functions of the state of a dynamical system, evolve along its trajectories.  Suppose the inputs $u(t_j)$ are bounded observations of invertible autonomous dynamics on an invariant set. The discrete-time Koopman operator $\mathcal{K}$ advances observables $\psi(u(t_j))$ forward by one sampling interval $h$ to $\psi(u(t_{j+1}))$ while the backward Koopman operator $\mathcal{K}^{-1}$ describes their evolution backwards in time. Under the stated assumptions, the current state of the observed system determines the entire past input history. Consequently, exponentially filtering this history defines a vector of state observables. To see this explicitly, observe that the filter \eqref{eq:b1} includes only inputs after $t_0 = 0$. Extending the sum backwards to infinity and subtracting the contribution from inputs at and before $t_0$ gives the identity 
{\small 
\begin{equation*} 
b_1(t_j) =\frac{1-\alpha}{1-\omega}
\sum_{k=0}^{\infty}\alpha^k u(t_j-kh)\ 
-\alpha^j\frac{1-\alpha}{1-\omega}
\sum_{k=0}^{\infty}\alpha^k u(-kh).
\end{equation*}
}
\noindent 
The first term equals $\psi(u(t_j))$. The second term is the transient due to the zero initialization of the reservoir; its sum is bounded and independent of $j$, so its contribution decays as $\alpha^j$. Thus $b_1(t_j)=\psi(u(t_j))$ up to an exponentially decaying error, and $A_b$ is the least-squares fit of the one-step backward evolution $\psi(u(t_{j-1}))=(\mathcal{K}^{-1}\psi)(u(t_j))$, i.e., an extended DMD approximation of $\mathcal{K}^{-1}$ with dictionary $\psi$ \cite{williams2015data}. When initialization effects are neglected (as is standard, since initial transients are discarded before training), the eigenvalues $\mu_i$ of $A_b$ are exactly the reciprocals of Koopman eigenvalues whenever the span of the components of $\psi$ is $\mathcal{K}$-invariant, and approximate them otherwise \cite{tu2014DMD,williams2015data}, while the eigenvectors of $\mathcal{K}^{-1}$ and consequently of $A_b$ coincide with those of the forward Koopman operator $\mathcal{K}$. The right eigenvectors $y_i \in \mathbb{R}^d$ of $A_b$ are the corresponding (approximate) Koopman modes of the observable $\psi$, and their images $v_i = W_{\mathrm{in}} y_i \in \mathbb{R}^n$ are the associated activity patterns across reservoir states. 

\subsection{Feedback Perspective on Reservoir Dominance}
\label{subsec:feedback_perspective}

Finally, we provide an input-output characterization of the reservoir training, which offers a complementary interpretation for the emergence of dominant invariant subspaces.
Closing the loop of the 
RC system introduces an output-feedback term $u(t) = y(t) = W_{\mathrm{out}} r(t)$.
In the spirit of root-locus analysis, the  closed-loop spectrum can be graphically visualized as the output weight matrix $W_{\mathrm{out}}$ changes during training.
We illustrate this idea for 
a multi-input multi-output (MIMO) reservoir, using trajectories generated by the Goldbeter system as training data.
The readout matrix is numerically recomputed on each data prefix
$\{t_1,\dots,t'\}$, producing $W_{\mathrm{out}}(t')$.
This yields
\begin{equation*}
J_{\mathrm{T}}(t')=J_0+W_{\mathrm{in}} \, W_{\mathrm{out}}(t'),
\qquad t'=t_1,\dots,t_m,
\end{equation*}
with $t_j = j h$,
and the plot of $\lambda[J_{\mathrm{T}}(t')]$ visualizes how the closed-loop spectrum evolves as more training data is used.

\textit{Example (MIMO Goldbeter input).}
We simulate the classical Goldbeter oscillator system \cite{Goldbeter1995}.
The Goldbeter system has $d=5$ state variables, so $W_{\mathrm{in}} W_{\mathrm{out}}$ has rank at most $5$. In simulation, at most five eigenvalue branches of $J_{\mathrm{T}}(t')$ separate from the open-loop cluster $\lambda[J_0]$, while the remaining $n-5$ eigenvalues stay fixed (Fig.~\ref{fig:rootLocus_linearRC_MIMO_Wout-t}).

Fig.~\ref{fig:rootLocus_linearRC_MIMO_Wout-t} also illustrates convergence to the infinite-data limit described by 
$C_{rr}=W_{\mathrm{in}}C_{b_1b_1}W_{\mathrm{in}}^\top$, where $C_{b_1b_1} = \frac{1}{T} \int_0^T b_1(t)b_1(t)^\top\,dt$ has rank $d=5$ as the Goldbeter example. Since $W_{\mathrm{in}}$ here has full column rank,
$\operatorname{rank}(C_{rr})=d=5$.
For finite training time, the covariance identity holds approximately, with an error that reduces as $T$ increases.  Fig.~\ref{fig:rootLocus_linearRC_MIMO_Wout-t} shows $q=5$ eigenvalues approaching the imaginary axis as more training data is used. Some initially overshoot and return, while the real eigenvalue continues to approach the imaginary axis.
The real part of the eigenvalues of the fully trained reservoir are very close to $0$, but still negative. This suggests that, in the linear setting, training drives the dominant mode as close as possible to the stability boundary while still maintaining stability. Intuitively, if these eigenvalues were to cross into the positive right half-plane, their associated modes would grow exponentially, causing instability 
and failing to match the (bounded) training data. Furthermore, in simulation we observe that even after adding  a redundant extra input channel that duplicates one of the input streams, the shift matrix $W_{\mathrm{in}} W_{\mathrm{out}}(t)$
has rank $5$, and five eigenvalues move away from the open-loop spectrum. Therefore the dominance of the trained reservoir captures the rank of the limiting reservoir second-moment matrix, i.e. the dimension of the reservoir subspace excited by the training data.

\begin{figure}
    \centering
    \includegraphics[width=1\linewidth, trim=0 0mm 0 0mm, clip]{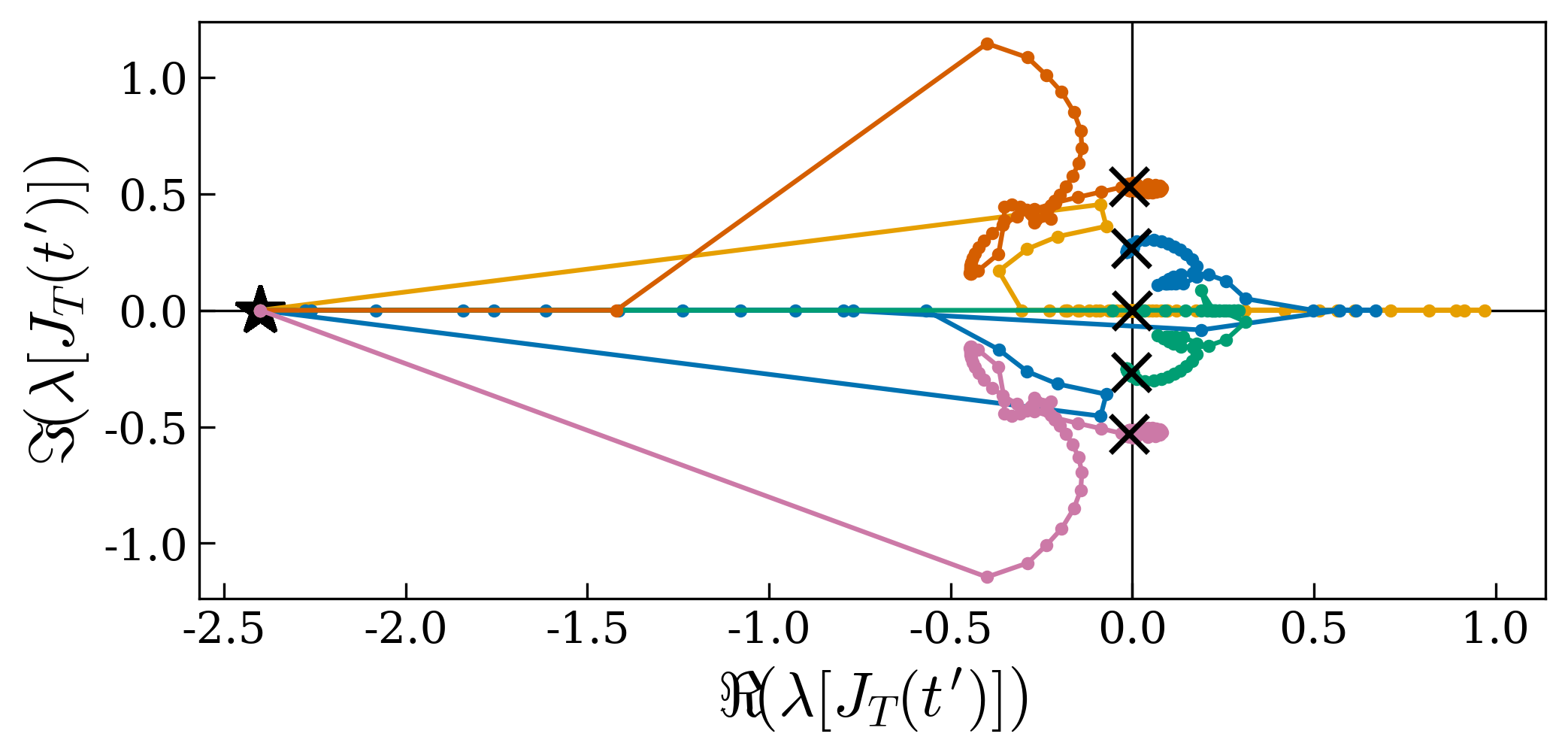}
    \vspace{-3mm}
    \caption{%
    Eigenvalues of $J_{\mathrm{T}}(t')$ for the diagonal linear reservoir ($W=\omega I$) trained on Goldbeter data ($d=5$). The black star marks the open-loop eigenvalue $\gamma(\omega-1)=-2.4$ (all $n$ eigenvalues of $J_0$ coincide). Colored paths show the five shifted eigenvalues as $t'$ increases, with a black $\times$ at the last training time. Parameters: $n=100$, $h=0.167$, $m=630$, $\gamma=11.99$, $\sigma_b=0$, $\omega=0.8$, $\beta=0$. 
    }
    \label{fig:rootLocus_linearRC_MIMO_Wout-t}
\end{figure}

\section{Towards Understanding Dominance of Nonlinear Reservoirs} 
\label{sec:nonlinear_RC}

Dominance analysis offers a natural generalization of our results to nonlinear reservoirs. 
Recall that, 
in the linear case, training adds to the fixed pre-training state matrix the low-rank term $W_{\mathrm{in}}W_{\mathrm{out}}$, so the spectrum of the trained reservoir follows directly from that of $W_{\mathrm{out}}W_{\mathrm{in}}$ by Theorem~\ref{thm:trained_reservoir_trajectories}. 
In a nonlinear reservoir, the Jacobian is state-dependent, so spectral analysis at a single operating point is insufficient. Differential dominance theory \cite{forni2019differential} provides the appropriate framework: it characterizes dominance through the tangent dynamics along trajectories, using the prolonged system~\eqref{eq:prolonged_system}.


Consider the reservoir \eqref{eq:RC_before_training} with $f=\tanh$, $W=\omega I$ with $\omega<1$, $\sigma=\sigma_b\mathbf{1}_n$, $\gamma=1$, and $W_{\mathrm{in}}$ of full column rank, together with its trained closed loop \eqref{eq:RC_after_training}. 
The Jacobian of the reservoir before training is
\begin{equation}
    J_{\text{0}}(r) = -I + \omega  \operatorname{diag}\!\Big(\operatorname{sech}^{2}\!\big(\omega r+\sigma_b\mathbf{1}_n\big)\Big).
    \label{eq:nonlin_J_bef}
\end{equation}
Since $\operatorname{sech}(x) \in (0,1]$ and $\omega<1$, eigenvalues of \eqref{eq:nonlin_J_bef} at any point $r \in \mathbb{R}^n$ are in the left half-plane.
After closing the loop, the Jacobian is
\begin{equation}
    J_{\text{T}}(r) = -I + \omega  \operatorname{diag}\!\Big(\operatorname{sech}^{2}\!\big(\omega r + \sigma_b\mathbf{1}_n\big) \Big)
    + W_{\text{in}} W_{\text{out}}.
    \label{eq:nonlin_J_T}
\end{equation}
As in the linear case, the low-rank term $W_{\mathrm{in}}W_{\mathrm{out}}$ added by the training generically shifts a subset of the Jacobian eigenvalues toward the right half-plane. Moreover, the bounded nonlinearity enables the dominant eigenvalues to cross far into the right half-plane, beyond marginal stability.
Numerically we see this is also true for a reservoir with a general matrix $W$ (Fig.~\ref{fig:Evals_linear_vs_nonlinear}).
Notice that the number of dominant modes is not restricted to the input dimension, as shown in Theorem \ref{thm:lin_res_dom_subspace}.
This yields an improved forecast horizon compared to the linear case (Fig.~\ref{fig:pred_linear_nonlinear_RC}).
%
%

\begin{figure}
\centering
\includegraphics[width=1\linewidth, trim=0 18.5mm 0 20mm, clip]{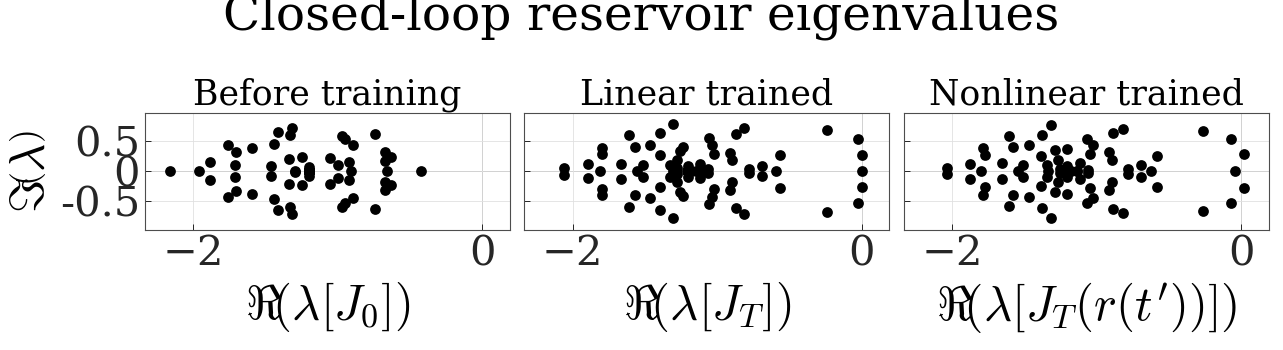}
\vspace{2mm}
\includegraphics[width=1\linewidth, trim=0 0mm 0 28.7mm, clip]{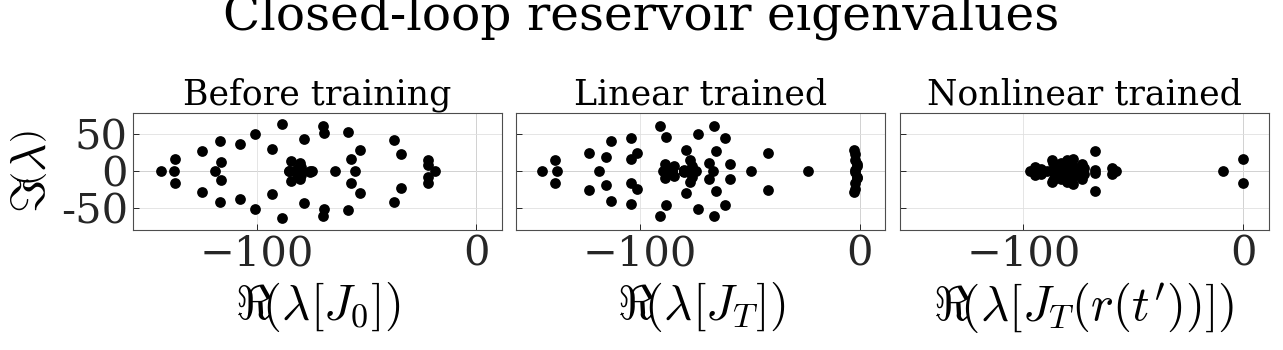}
\vspace{-8mm}
\caption{
Comparison of the evolution of RC spectrum (black points) with a general reservoir matrix $W$, at some point $r_{\mathrm{fixed}}$ for: Goldbeter oscillator (top) vs Lorenz attractor with $d=3$ (bottom), before training (left), linear RC after training (middle), and nonlinear RC after training at some time $t'$ (right).
The reservoirs were initialized and trained with the same parameters as in Fig. \ref{fig:rootLocus_linearRC_MIMO_Wout-t} for Goldbeter model with $m_f=350$, and with $n=100$, $h=0.01$, $m=10000$, $m_f=600$, $\gamma=80$, $\sigma_b=0.6$, and $\omega=0.8$ for the Lorenz model.
Both models use $\beta=10^{-4}$.
}
\label{fig:Evals_linear_vs_nonlinear}
\end{figure}

\begin{figure}
\centering
\begin{tabular}{@{}c@{\hspace{0mm}}c@{}}
\raisebox{3.4cm}{(a)} & 
\includegraphics[width=0.97\linewidth, trim=6mm 11mm 0 9mm, clip]{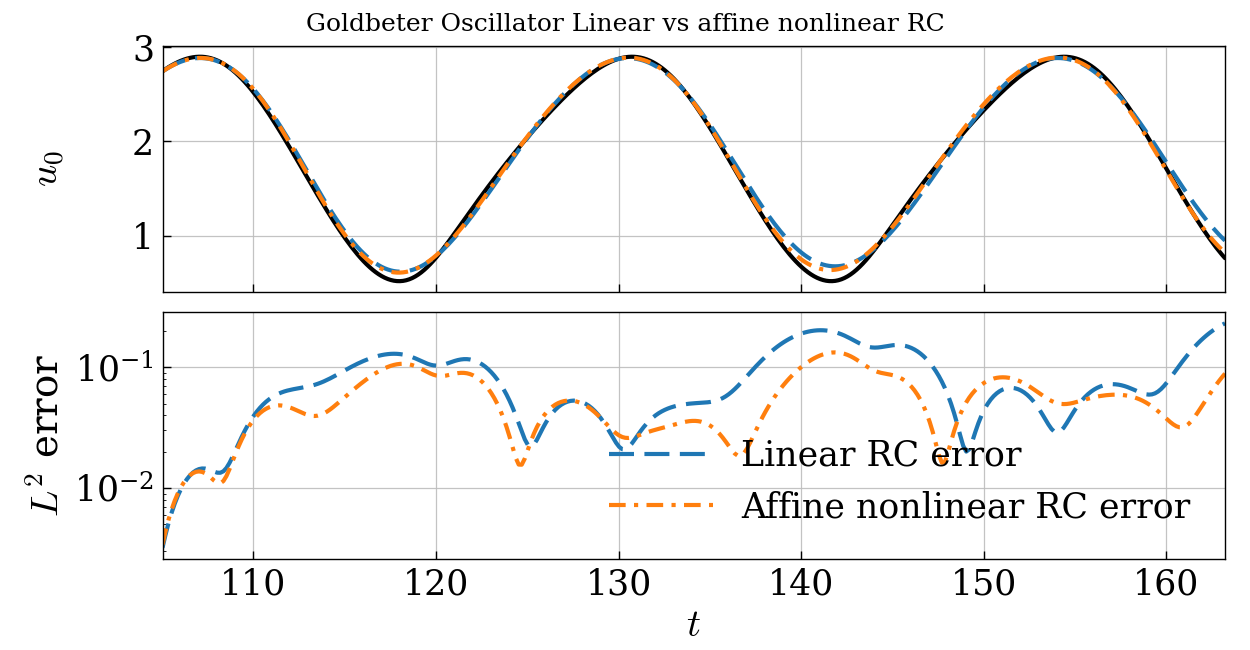}
\end{tabular}
\vspace{2mm}
\begin{tabular}{@{}c@{\hspace{0mm}}c@{}}
\raisebox{3.6cm}{(b)} & 
\includegraphics[width=0.97\linewidth, trim=6mm 0mm 0 9mm, clip]{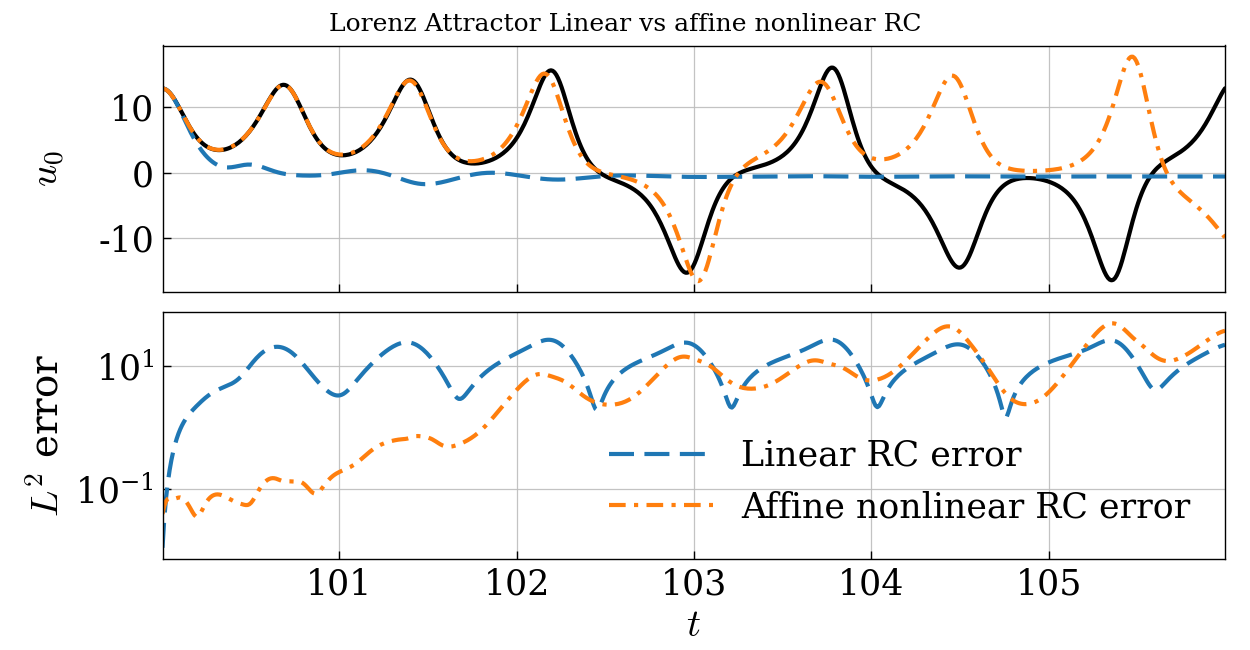}
\end{tabular}
\vspace{-4mm}
\caption{
Comparison of the prediction of linear (blue dashed) and nonlinear (orange dotted) RC with a general reservoir matrix $W$, of the Goldbeter oscillator (a) and Lorenz attractor (b), with ground truth (black), for the system variable $u_0$, together with the corresponding $L^2$ prediction error $\|u(t)-y(t)\|_2$ over all system variables. 
The reservoirs were initialized and trained with the same parameters as in Fig.~\ref{fig:Evals_linear_vs_nonlinear}.
}
\label{fig:pred_linear_nonlinear_RC}
\end{figure}

\subsection{Differential Dominance Analysis}

We now connect the linear and nonlinear analyses using $p$-dominance. In both settings, the key observation is that training adds the low-rank feedback term $W_{\mathrm{in}}W_{\mathrm{out}}$ to the open-loop dynamics, and this term is responsible for creating a spectral splitting between dominant and transient directions.

In the linear case, $p$-dominance of the trained system \eqref{eq:RC_linear_trained} with rate $\lambda \geq 0$ is equivalent to the existence of a symmetric matrix $P$ with inertia $(p,0,n-p)$ satisfying
\begin{equation*}
(J_0 + W_{\mathrm{in}}W_{\mathrm{out}})^\top P + P (J_0 + W_{\mathrm{in}}W_{\mathrm{out}}) \leq -2\lambda P -\varepsilon I,
\end{equation*}
for some $\varepsilon>0$, as in Definition~\ref{def:p_dom}.
In the nonlinear case, differential $p$-dominance of 
the trained system
requires the same inequality to hold pointwise: there must exist a constant $P$ with inertia $(p,0,n-p)$ such that, for all $r \in \mathbb{R}^n$,
\begin{equation*}
(J_0(r) + W_{\mathrm{in}}W_{\mathrm{out}})^\top P + P (J_0(r) + W_{\mathrm{in}}W_{\mathrm{out}}) + 2\lambda P \leq -\varepsilon I,
\end{equation*}
as in Definition~\ref{def:diff_p_dom}. This is more demanding than the linear condition, since $J_0(r)$ varies with $r$ and the inequality must hold uniformly. However, the boundedness of $\operatorname{sech}^2(\cdot)$ in \eqref{eq:nonlin_J_bef} confines $J_0(r)$ to a compact set, which makes the uniform condition potentially tractable. In both cases, training can be interpreted as a feedback design mechanism that reshapes the spectrum inducing a splitting between dominant and transient modes. Conditions under which training provably induces $p$-dominance are the subject of current research.

\section{Discussion}
\label{sec:discussion}

In the presented work we characterized how low-rank feedback learned during training generates dominant subspaces in linear reservoirs. For general linear reservoir constructions, we showed  that training induces dominant modes in the reservoir whose eigenvalues settle on the imaginary line in the infinite data limit. We then proved equivalence between the dimension of the emergent dominant subspaces and the rank of the limiting reservoir second-moment matrix. For diagonal reservoirs, we derived explicit spectral formulas and an exact backward DMD representation, connecting the trained dominant modes to a finite-dimensional Koopman approximation. In future work, we will expand the connection to Koopman theory beyond diagonal reservoirs, and extend our analysis to nonlinear reservoirs through differential $p$-dominance.

\section*{Acknowledgements}
The authors acknowledge the use of OpenAI ChatGPT Astra for copy editing and feedback on correctness of mathematical results. 
The authors take full responsibility for the accuracy, originality, and integrity of 
the manuscript.

\bibliographystyle{./bibliography/IEEEtran}
\bibliography{./bibliography/references}

\end{document}